\def\year{2020}\relax
\documentclass[letterpaper]{article} 
\usepackage{amsfonts}
\usepackage{multirow}
\usepackage{amsthm}
\usepackage{mathtools}
\usepackage{bbm}
\usepackage{algorithm}
\usepackage[noend]{algpseudocode}
\usepackage[font=footnotesize]{subcaption}
\algnewcommand{\Input}[1]{%
	\State \textbf{Input:}
	\Statex \hspace*{\algorithmicindent}\parbox[t]{.8\linewidth}{\raggedright #1}
}
\algnewcommand{\Output}[1]{%
	\State \textbf{Output:}
	\Statex \hspace*{\algorithmicindent}\parbox[t]{.8\linewidth}{\raggedright #1}
}
\usepackage{aaai20}  
\usepackage{times}  
\usepackage{helvet} 
\usepackage{courier}  
\usepackage[hyphens]{url}  
\usepackage{graphicx} 
\usepackage{tikz}
\usetikzlibrary{patterns,decorations.pathreplacing}
\usetikzlibrary{quotes,arrows.meta}
\usetikzlibrary{arrows}
\definecolor{myBrown}{RGB}{139,69,19}
\definecolor{myYellow}{RGB}{244,160,0}
\definecolor{greenlayer}{RGB}{147, 180, 63}
\definecolor{lmuGruen}{RGB}{63, 180, 136}
\usepackage{graphicx}  
\newtheorem{lemma}{Lemma}
\newtheorem{theo}{Theorem}
\nocopyright 

\title{MMGAN: Generative Adversarial Networks for Multi-Modal Distributions }
\author{Teodora Pandeva and Matthias Schubert\\
Institute for Informatics\\
LMU Munich\\
Teodora.Pandeva@campus.lmu.de and schubert@dbs.ifi.lmu.de}
\begin{document}

\maketitle

\begin{abstract}
Over the past years, Generative Adversarial Networks (GANs) have shown a remarkable generation performance especially in image synthesis. Unfortunately, they are also known for having an unstable training process and might loose parts of the data distribution for heterogeneous input data. In this paper, we propose a novel GAN extension for multi-modal distribution learning (MMGAN).  In our approach, we model the latent space as a Gaussian mixture model with a number of clusters referring to the number of disconnected data manifolds in the observation space, and include a clustering network, which relates each data manifold to one Gaussian cluster. Thus, the training gets more stable. Moreover, MMGAN allows for clustering real data according to the learned data manifold in the latent space. By a series of benchmark experiments, we illustrate that MMGAN outperforms competitive state-of-the-art models in terms of clustering performance.
\end{abstract}

\section{Introduction}
 Generative Adversarial Nets (GANs) \cite{goodfellow2014generative} are state-of-the-art deep generative models and, therefore, they are primarily designed to model data distributions. 
 Compared to other generative models, GANs gain distinction in generating higher quality data.  Despite their notable success, GANs still suffer from unsolved problems and thus, there is ongoing research to further improve their performance and make training more stable. For instance, GAN implicit nature  does not allow to apply inference learning on the latent space. Although many methods exist which deal with this shortcoming most of them lack interpretability of the estimated posterior distribution. Moreover, GANs model the latent space as a simple unimodal distributions, ignoring the often more complicated implicit structure of the learned data distribution. However, for many data sets, a union of disjoint manifolds (or clusters) fits more naturally to the implicit structure of the input data. For example, digit data can be interpreted as samples from a disjoint union of $10$ manifolds -- one for each digit. \cite{Khayatkhoei18Disconnected} showed that the quality of generated data suffers from the generator attempting to cover all data manifolds in the data  space with a single manifold in the latent space. Hence, this can lead to mode dropping, i.e. one or more submanifolds of the real data are not covered by the generator. It has been proven that GAN's local convergence can be sustained when the real and fake data distribution are near achieving a Nash equilibrium \cite{nagarajan2017gradient}. Thus, ignoring the multi-modal nature of a data set might lead to oscillating generator parameters without converging to the real distribution.
  
  In this paper, we introduce GANs for learning multi-modal distributions. The resulting architecture, which is named MMGAN, adopts a dynamic disconnected structure of the latent space, which is distributed according to a Gaussian mixture model. More precisely, by introducing an extra network into the GAN structure, the resulting framework aims to find a disconnected data representation in the latent space, such that each data mode or cluster in the observation space is related to a single cluster in the latent space. This stabilizes the training process and yields a better data representation. Furthermore, we can do inference on the real data to predict the most likely cluster in the latent space and thus, categorize the data with respect to its implicit structure. We provide an universal approximation theorem assuring the existence of a generator with the MMGAN functionality in the spirit of \cite{Cybenko89,Hornik91}.
  
  
\section{Related Work}
\label{sec:related}
 There is a great variety of GAN architectures, which explore the latent space abilities to produce realistic data. Most of them can be referred to as hybrid VAE-GAN methods, which bridge the gap between Variational Autoencoders (VAEs) \cite{kingma2013auto} and GANs. All of them use a third encoder network which maps a data object to an instance from the latent space $\mathcal{Z}$. For example,  \citeauthor{makhzani2015aae} proposed the Adversarial Autoencoder (AAE), which is an autoencoder for performing inference.  AAE is composed of three networks: encoder, decoder and discriminator. The latter is trained to correctly classify an encoded noise  from the prior noise, which is an arbitrary noise distribution. Although this model can be extended to learn a discrete data representation in an unsupervised learning fashion, it does not consider the true data distribution. Another VAE-GAN  hybrid is ClusterGAN \cite{Mukherjee2018clustergan}, which is essentially InfoGAN \cite{chen2016infogan} followed by k-means post-clustering on the encoded latent codes.  Although InfoGAN has shown remarkable generation and clustering performance by semantically disentangling the latent space, we argue that  the encodings are not suitable for $k$-means clustering, which tends to discover spherical patterns in the data. 
 
 A further approach for gaining more insights into the structure of the noise $\mathcal{Z}$, is to directly model the latent space by imposing some assumptions about the prior distribution $\mathbb{P}_z.$ For instance, GM-GAN \cite{yosef2018gmgan} and DeLiGAN \cite{gurumurthy2017deligan} adopt a Gaussian mixture model for the latent space distribution, where the means and standard deviations are learnable parameters. However, these models do not provide any direct inference framework and any interpretation of the learned latent components.
 
The Gaussian Mixture VAE (GMVAE) is adapted for unsupervised clustering tasks as the latent space $\mathcal{Z}$ attains the form of Gaussian mixture model \cite{DilokthanakulMG16}. Thus, it becomes the explicit counterpart of MMGAN. However, based on the VAE framework, GMVAE has shown some shortcomings, including the VAE's tendency to produce blurry images and the (strong) restriction of the encoder output distribution.
  
\section{Generative Adversarial Networks}
\label{sec:GAN}
GANs consist of two networks which are opposed to one another in a  game \cite{goodfellow2014generative}. The first one, $G$, is a \textbf{generator}, which  captures the data distribution and tries to produce realistic data. It receives as input noise data, sampled from the latent space $\mathcal{Z}\subset \mathbb{R}^d$ with dimension $d$, which is deterministically mapped  to a data point from the observation space $\mathcal{X}\subset \mathbb{R}^p$,  where $p\geq d$.  The second player, $D$, is called \textbf{discriminator}. It measures how realistic the input data is, i.e. for some $x\in\mathcal{X},$ $D(x)$ is in general a score, e.g. the probability, measuring whether $x$ comes from the real distribution. Thus, $D$ is trained via  supervised learning on data with assigned labels $1$ for being real and $0$ for being fake and tries to correctly classify a real object from a fake one. In contrast, $G$ aims to fool the discriminator by producing data resembling the real data as close as possible.

In this setting, the models $D$ and $G$ are neural networks with fixed structures. Hence, the learning takes place over the networks parameters, which are denoted by  $\theta_D\in\Theta_D$ and $\theta_G\in\Theta_G$, respectively, where $\Theta_D$ and $\Theta_G$ are real spaces with dimension depending on the networks architectures.  For simplicity, throughout this work, we use $G(z)$ and $D(x)$ instead of $G(z;\theta_G)$ and $D(x;\theta_D),$ respectively, unless it is explicitly mentioned. Here, the unknown true data distribution is denoted by $\mathbb{P}_{r},$   defined on $\mathcal{X}$,  and $\mathbb{P}_z$ is the input noise distribution, defined on $\mathcal{Z}$. Given a noise instance $z$, the generator produces a fake data point $x_f=G(z),$ which is a sample of the unspecified distribution, induced by $\mathbb{P}_f$, which is the implicit approximation of $\mathbb{P}_r$ and gives the methodology how $\mathcal{Z}$ is related to $\mathcal{X}$.  

The standard GAN optimization problem (SGAN) \cite{goodfellow2014generative} is defined by
\begin{equation*}
\begin{split}
\hat{\theta}_D,\hat{\theta}_G&=\arg\min_{\theta_D}\arg\max_{\theta_G}V(D,G) \text{ with }\\
V(D,G)&=-\mathbb{E}_{x_r \sim \mathbb{P}_{r}}\big[\log(D(x_r))\big]\\&-\mathbb{E}_{x_f \sim \mathbb{P}_f}\big[\log(1-D(x_f))\big].
\end{split}
\end{equation*}
 \citeauthor{jolicoeur2018relativistic} gives a theoretical and empirical analysis of the SGAN training behavior which contradicts the theoretical results, derived by \cite{goodfellow2014generative}, i.e. the probability of real data being real should decrease during training, while the probability of fake data being fake should increase, which is not fulfilled by SGAN. To excel the training stability, \textit{relativistic} objective functions are proposed \cite{jolicoeur2018relativistic}. A class representative is the RSGAN, defined in the initial paper. The corresponding optimization problem is given by
 \begin{equation*}
 \begin{split}
 \hat{\theta}_D&= \arg\min_{\theta_D}- \mathbb{E}_{x_r\sim\mathbb{P}_{r},x_f\sim\mathbb{P}_f}[\log(s (C(x_r)-C(x_f))]\\
 \hat{\theta}_G&= \arg\min_{\theta_G}- \mathbb{E}_{x_r\sim\mathbb{P}_{r},x_f\sim\mathbb{P}_f}[\log(s (C(x_f)-C(x_r))],
 \end{split}
 \end{equation*}
 where $C(\cdot)$ is the critic of the discriminator, which is defined by  the non-transformed output of $D$  \cite{arjovsky2017wasserstein}, i.e. $C(x)=\text{logit}(D(x))$, and $s:\mathbb{R}\cup \{-\infty,+\infty\}\rightarrow [0,1],\, s(x)=(1+\exp(-x))^{-1}$ is the sigmoid function. This is equivalent to the negative expected probability that a real data object is more realistic than a fake one. A further example of the relativistic approach is the relativistic average GAN (RaSGAN), which compares a real object to the average fake one and vice versa. We provide a modified version of RaSGAN in the next section. Notably, this family of objective functions allows a direct comparison between pairs of fake and real data objects. This is a key feature which is utilized in the MMGAN structure.

\section{Multi-Modal GANs (MMGAN)}
\label{sec:mmgan}
MMGAN samples noise from a Gaussian mixture model. In particular, we sample a  cluster from a cluster distribution $\operatorname{Cat}\Big(K,\frac{1}{K}\Big)$ and then, draw the noise $z$ from the Gaussian corresponding to the sampled cluster. As any GAN, MMGAN takes $z$ as input to a generator ($G$) and employs a discriminator ($D$) to guide $G$ to generate realistic objects. In addition, MMGAN employs an encoder network ($E$) to predict the cluster of data objects. The output is a probability distribution over the clusters, which is denoted by $p_E(\cdot\vert x)$ for $x\in\mathcal{X}$. Thus, the encoder $E$ should reproduce the cluster from which a fake instance was sampled from, and predict the most likely cluster for real images. An overview of the architecture is depicted in Figure \ref{fig:mmgan}.

\begin{figure}
	\centering
\includegraphics[scale=0.7]{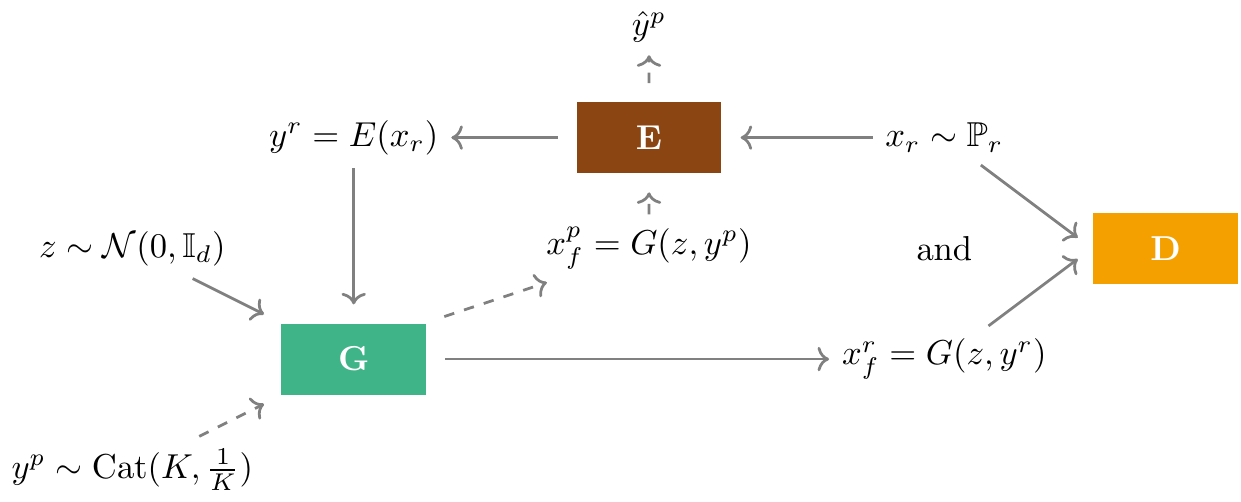}
	\caption[MMGAN Architecture]{MMGAN architecture consists of three networks: generator, discriminator and encoder. There are two training cycles. The solid lines show the data flow for building  pairs between a real and a fake sample. The discriminator $D$ is trained on the formed pairings (batch $B1$). The dashed lines follow the generation of prior data from a randomly sampled cluster (batch $B2$). Thus, the encoder $E$ learns to correctly classify an $x_f^p$ with a true label $y^p.$ }
	\label{fig:mmgan}
\end{figure}

In the following, we will describe our architecture in more detail. To force MMGAN to cluster data, we model the latent space as a mixture of Gaussians with a uniform prior over the clusters. Thus, our goal is to find a representation of each cluster in terms of mean and covariance. We restrict the covariance matrix to have equal diagonal entries and $0$ everywhere else, i.e. to be of the form $\sigma^2 \mathbb{I}_d$  where $\sigma \in\mathbb{R}$. Therefore, we define the pairing $(\mu_k,\sigma_k)$ with $\mu_k\in \mathbb{R}^d$ and $\sigma_k\in\mathbb{R}$ to be the mean and standard deviation of the $k$th cluster, where for all $k \in\{0,\ldots,K\}$,  $(\mu_k,\sigma_k)$ are learnable parameters. Both parameters are represented by a $K\times d$ and $K\times 1$ dense layers located right before the core generator, denoted by $\mu$ and $\sigma$, respectively. Both networks receive a one hot encoded cluster as input, i.e. of the form $y=(0,\ldots,1,\ldots, 0),$ and output $\mu_k:=\mu(y)$ and $\sigma_k:=\sigma(y)$ for the $k^{th}$ entry of $y$ being $1$, i.e. $y_k=1$. Thus, for the cluster related noise $\tilde{z}$,  we obtain $\tilde{z}=\mu(y)+\sigma(y)z$ using the reparameterization trick \cite{kingma2013auto}.  Afterwards $\tilde{z}$ is fed into the core generator. The exact procedure is illustrated in Figure \ref{fig:gen}.
To formalize the whole procedure, we describe the generative model $\mathbb{P}_f$ as:
\begin{align*}
\begin{split}
y&\sim \operatorname{Cat}\Big(K,\frac{1}{K}\Big),\\
\tilde{z}\vert y &\sim  \mathcal{N}(\mu(y),\sigma(y)^2\mathbb{I}_d),\\
x_f&=G(\tilde{z}),
\end{split}
\end{align*}

	Doing inference on the latent space can be done by directly computing the posterior $p(z\vert x)$ for $z\in\mathcal{Z}$ and $x\in\mathcal{X}$, i.e.
	\begin{align*}
	p(z\vert x)=\sum_{k=1}^K p_E(y\vert x)\mathcal{N}(z\vert\mu_k,\sigma_k^2\mathbb{I}_d),
	\end{align*}
	where $\mathcal{N}(z\vert\mu_k,\sigma_k^2\mathbb{I}_d)$ denotes the probability distribution function of $\mathcal{N}(\mu_k,\sigma_k^2\mathbb{I}_d)$.

\begin{figure}[t]
	\centering
\includegraphics[scale=0.7]{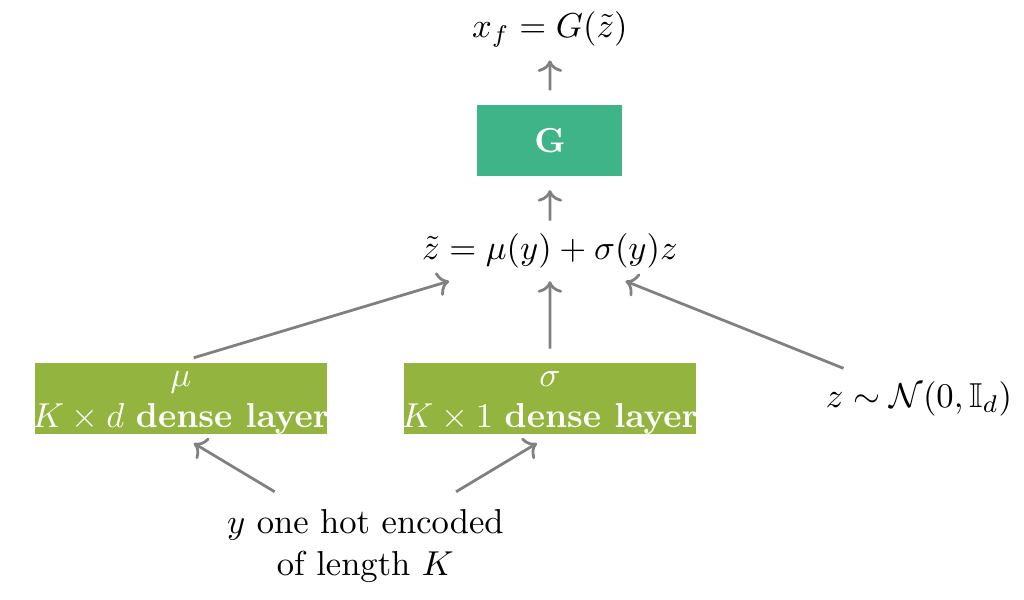}
	\caption[Generator Architecture]{The generator is extended by two layers representing the cluster mean and standard deviation. Their input is the cluster which is one hot encoded. Hence, the noise $z$ is transformed into a Gaussian variable with mean $\mu(y)$ and $\sigma(y)^2\mathbb{I}_d$ variance. }
	\label{fig:gen}
\end{figure}	

\subsection{MMGAN Training}
To explain the training of MMGANs, we will start with the generation of training batches. Firstly, we sample real objects $x_r$ from the unknown distribution $\mathbb{P}_r.$ These are fed into the encoder $E$, such that the resulting output is transformed into a cluster $y^r$, which refers to a Gaussian cluster in the mixture model. The encoded $y^r$ together with randomly sampled standard Gaussian noise serve as an input to the generator $G$. Thus, fake objects $x_f^r$ are produced from the Gaussian cluster, corresponding to the sampled real data $x_r$. The resulting pairings $(x_r,x_f)$ are fed into the discriminator $D$.  We argue that this pairing system will excel both the data generation process and the clustering performance. We will refer to this training batch as \textbf{B1}.

In addition, to train the encoder $E$ to correctly assign clusters to fake objects, we generate a second type of training batch, \textbf{B2}, which is composed of fake observations $x_f^p$ labeled by their corresponding clusters $y^p$. For this batch, the clusters $y^p$ are randomly drawn samples from the categorical distribution $\operatorname{Cat}(K,\frac{1}{K})$. 

\begin{algorithm}[t]
	\caption{MMGAN}\label{supgmgan}
	\begin{algorithmic}[1]
		\Input{$K$: number of clusters\\
			\textit{train\_iter}: number of training iterations\\
			$m$: batch size\\
			$d$: noise dimension\\
			$\alpha$: hyperparameter\\
		}   \vspace{0.3cm}
		\State Initialize $\theta_D^0,\theta_G^0,\theta_E^0$.
		\For {$t=1$ to  \textit{train\_iter}}
		\State Sample $x_r$ from data of size $m$.
		\State $y=E(x_r;\theta_E^{t-1})$ 
		\State Sample $y^p\sim \text{Cat}(K,\frac{1}{K})$ of size $m$.
		\State Sample $z\sim \mathcal{N}(0,\mathbb{I}_d)$ of size $m$.
		\State $x_f=G(z,y;\theta_G^{t-1})$ 
		\State $x_f^p=G(z,y^p;\theta_G^{t-1})$  
		\State $l_D^{t-1}=-\frac{1}{m} \Big(\sum_{i=1}^{m}\nabla_{\theta_D}\log(s(C(x_r^{(i)};{\theta_D^{t-1}})-C(x_f^{(i)};{\theta_D^{t-1}})))\Big)$
		\State $l_{G,E}^{t-1}= -\frac{1}{m} \Big(\sum_{i=1}^{m}\nabla_{(\theta_G,\theta_E)}\log(s(C(x_f^{(i)})-C(x_r^{(i)})))+\alpha\log p_E(y^{p(i)}\vert x_f^{p(i)})\Big)$
		\State Update $\theta_D^{t-1}$ by Adam  with gradient $l_D^{t-1}$.
		\State Update $(\theta_G^{t-1},	\theta_E^{t-1})$ by Adam with gradient $ l_{G,E}^{t-1}$.
		\EndFor
		\Output{$D,G,E$}
	\end{algorithmic}
	
\end{algorithm}

The exact optimization problem for training MMGANs is given as follows:

\begin{align*}
\begin{split}
\min_{\theta_D} \max_{\theta_G,\theta_E}& V(D,G,E)\\
&+\alpha\mathbb{E}_{y\sim \text{Cat}(K,\frac{1}{K}),z\sim \mathcal{N}(0,\mathbb{I}_d),x_f =G(z,y)}[\log p_E(y\vert x_f)],
\end{split}
\end{align*}
where $ V(D,G,E)$ refers to an adversarial loss, depending on the three nets, and $p_E(\cdot\vert x)$ is the output encoder posterior for given observation $x\in\mathcal{X}$. Note that $V(D,G,E)$ is trained on the first type of batch \textbf{B1} and thus, depends on $E$ to generate fake instances $x_f$.
The second term is the cross entropy loss for the encoder output, weighted by a hyperparameter $\alpha>0$. This term is trained on the second type of batch \textbf{B2} to make sure that each cluster is sufficiently represented.

The training steps are shown in Algorithm \ref{supgmgan}. To lay emphasis on the dependence of $V(D,G,E)$ on the parameters $\theta_D,\theta_G$ and $\theta_E$, in Algorithm \ref{supgmgan}, the notation  $D(\cdot;\theta_D),\, G(\cdot;\theta_G)$  and $E(\cdot;\theta_E)$ is used. Here, the chosen adversarial loss is RSGAN.  Moreover, we use Adam \cite{kingmaB14adam} for parameter learning (see lines $12$--$13$ of Algorithm \ref{supgmgan}).

We choose $V(D,G,E)$ to be a relativistic objective since we aim to measure similarity between real and fake objects, belonging to the same cluster. If we assume that $E$ clusters data objects in a meaningful way, we expect that the discriminator will find it more difficult to classify a real object $x_r$ from a fake one $x_f$ from the same mode. Thus, we argue that the discriminator will not reach optimality very fast, which will lead to a more stable GAN training behavior \cite{arjovsky2017manifolds}.

In addition to using standard RSGAN loss, we propose an extension of the RaSGAN \cite{jolicoeur2018relativistic}, which is a cluster-wise comparison between a real data object and the average fake one or vice versa, i.e. $\hat{\theta}_D,\hat{\theta}_G$ and $\hat{\theta}_E$ are optimal solutions of the optimization problems
	\begin{align*}
	\begin{split}
	\hat{\theta}_D=\min_{\theta_D}& -\mathbb{E}_{x_r \sim \mathbb{P}_r}[\log \hat{ D}(x_r)]\\& - \mathbb{E}_{ z\sim\mathbb{P}_z,y\sim\operatorname{Cat}(K,\frac{1}{K})}[\log(1-\hat{D}(G(z,y)))],\\
	\hat{\theta}_G,\hat{\theta}_E=\min_{\theta_G,\theta_E} - &\mathbb{E}_{ z\sim\mathbb{P}_z,y\sim\operatorname{Cat}(K,\frac{1}{K})}[\log \hat{ D}(G(z,y))]\\& - \mathbb{E}_{x_r \sim \mathbb{P}_r}[\log(1-\hat{D}(x_r))],
	\end{split}
	\end{align*}
	where
	\begin{align*}
	&\hat{D}(x_r)= s(C(x_r)-\mathbb{E}_{z\sim\mathbb{P}_z}[C(G(z,E(x_r)))])\\
	&\hat{D}(G(z,y))=s(C(G(z,y))-\mathbb{E}_{x_r\sim\mathbb{P}_r}[\delta_{E(x_r)}(y)\cdot C(x_r)]),
	\end{align*} 
	and  $\delta_{E(x_r)}(y)$ is the Dirac delta function with $\delta_{E(x_r)}(y)=1$ for $E$ assigning the highest probability of $x_r$ being in cluster $y$, and $0$ otherwise. We name the resulting objective conditional RaSGAN (cRaSGAN).

\subsection{Universal Approximation Theorem for the Latent Space Assumption}

In the following, we will guarantee the existence of a fully connected neural network that maps a collection of Gaussians to the disjoint data space such that the resulting network recovers the initial data distribution up to a constant $\epsilon$. This result is closely related to the Universal Approximation Theorem of \cite{Cybenko89,Hornik91}. A similar theory for the uniform distribution has been recently developed in the work  of \cite{khrukov2019universality}. 


Here, we are interested in \textit{smooth} (infinitely differentiable) functions $f$, which surjectively map the support of a Gaussian distribution to a $d$-connected manifold, as defined in \cite[ Definition $1.4.1$]{jost2008riemannian}. In real life applications, we choose the dimensionality of the latent space $\mathcal{Z}\subseteq\mathbb{R}^d$ to be very large, due to the high dimensionality of the observation space $\mathcal{X}.$ The Gaussian Annulus theorem (see \cite[Theorem $2.9$]{blum2015foundations}) suggests that for a large enough $d$, the mass of a Gaussian with zero mean and identity matrix as covariance is concentrated around the periphery of a a ball with radius $\sqrt{d-1}$ and, thus, it approximates the sphere  $\sqrt{d-1}\cdot \mathbb{S}^{d-1}=\{x\in\mathbb{R}^d\vert \Vert x\Vert_2^2=d-1\}$.  The theory developed below is adapted to high dimensional spherical latent spaces, since these can be easily extended to high dimensional Gaussians.

\begin{lemma}
	\label{lem:mapsphere}
	Let $\mathcal{M}\subset\mathbb{R}^p$ for $p\geq d$ be a compact connected $d$-dimensional manifold. Then there exists a
	smooth map
	\begin{align*}
	f:\mathbb{S}^{d}\rightarrow \mathbb{R}^p,
	\end{align*}
	such that $f(\mathbb{S}^{d}) = \mathcal{M},$ where  $\mathbb{S}^{d}=\{x\in\mathbb{R}^{d+1}\vert \Vert x\Vert_2=1\}$.
\end{lemma}
\begin{proof}[Proof.]
	We use  \cite[Theorem 5.1]{khrukov2019universality}, implying the existence of a surjective smooth map $g: \overline{\mathcal{B}_d^1(0)}\rightarrow \mathcal{M}$, where $\overline{\mathcal{B}_d^1(0)}$ is the  (closed $d$-dimensional) unit ball with origin $0,$ i.e. $\overline{\mathcal{B}_d^1(0)}=\{x\in\mathbb{R}^{d}\vert \Vert x\Vert_2\leq1\}.$ Now, we construct a smooth surjective function $\phi:\mathbb{S}^{d}\rightarrow \overline{\mathcal{B}_d^1(0)},$ such that the resulting map $f:=g\circ \phi,$ i.e  $f(x)=g(\phi(x))$ for all $x\in\mathbb{S}^{d}$, fulfills the above stated requirements. 
	
	Let $\phi:\mathbb{S}^{d}\rightarrow \overline{\mathcal{B}_d^1(0)}$ be the projection on the unit ball defined by $\phi(x)=(x_1,\ldots,x_d)^T$ for $x \in \mathbb{S}^{d}.$ This map is smooth because it can be represented in matrix form by $A=\mathrm{diag}(\underbrace{1,\ldots,1}_{d \text{ times}},0),$ such that $\phi(x)=Ax.$ It is also surjective because for each $y\in \overline{\mathcal{B}_d^1(0)}$ the point $x=(y_1,\ldots,y_p,(1-\sum_{i=1}^p y_i^2)^{1/2})$  fulfills $\phi(x)=y$ and is contained in $\mathbb{S}^{d}$ since $\Vert x\Vert_2=1.$ 
	
	Thus, the map $f:=\phi\circ g$   is smooth and surjective since it is a composition of smooth and surjective maps.
\end{proof}

 Let $\mathcal{Z}$  be the support of a $d+1$-dimensional Gaussian of the form $\mathcal{N}(0,\mathbb{I}_{d+1})$. Hence, the map $g:\mathcal{Z}\rightarrow \mathcal{M},\,g(x)=f\Big(\frac{x}{\Vert x\Vert_2}\Big),$  where $f$ is defined as in Lemma \ref{lem:mapsphere}, is smooth and $g(\mathcal{Z})=\mathcal{M},$ since $ \frac{x}{\Vert x\Vert_2} \in \mathbb{S}^{d}.$

According to the Universal Approximation Theorem by \cite{Cybenko89,Hornik91} the map defined in Lemma \ref{lem:mapsphere} can be approximated by a fully-connected neural network arbitrarily well. This observation is translated into the more general case of disconnected manifolds in Theorem \ref{theo:uni2}. In this setting,  the approximation error is measured by means of the Hausdorff distance $d_H$ \cite[p. 280]{munkres2017topology}, which is defined by \[d_H(X,Y)=\max\{\,\sup_{x \in X} \inf_{y \in Y} m(x,y),\, \sup_{y \in Y} \inf_{x \in X} m(x,y)\,\},\] where $m$ is a well-defined metric on $\mathbb{R}^d$ and  $X, Y\subseteq \mathbb{R}^d$. Thus, we aim to find a network $G$ such that the value for  $d_H(G(\mathcal{Z}),\mathcal{X})$ is kept to be low. 
\begin{theo}
	\label{theo:uni2}
	Let $\mathcal{X}=\bigcup_{i=1}^K\mathcal{X}_i \subset \mathbb{R}^p,\,p\geq d$ be a disconnected union of  compact connected $d$-dimensional manifolds. Then for every $\epsilon>0$ and every nonconstant, bounded, continuous activation function $\phi:\mathbb{R}\rightarrow \mathbb{R}$, there exists  a fully connected neural network $G:\mathbb{R}^{d+1}\rightarrow \mathbb{R}^p$  with activation function $\phi$ such that the following is fulfilled. 
	
	There exists a collection $\{S_i\}_{i=1}^K$ of disjoint $d+1$-dimensional compact annuli such that for all $i \in\{1,\ldots, K\}$
	\begin{align*}
	d_H(G(S_i),\mathcal{X}_i)<\epsilon.
	\end{align*}
\end{theo}

\begin{proof}[Proof.]
	The collection $\{S_i\}_{i=1}^K$ is constructed explicitly. Let $D=[-2\sqrt{d},2\sqrt{d}]^{d+1}$ be a $d+1$-dimensional cube and $\{v_1,\ldots,v_{2^{d+1}}\}$ be the set of all vertices, i.e. for all $ i\in\{1,\ldots, 2^{d+1}\}$, $v_i\in\{-2\sqrt{d},2\sqrt{d}\}^{d+1}$. We choose arbitrarily $K$ vertices  $\{v_{(1)},\ldots,v_{(K)}\}\subseteq \{v_1,\ldots,v_{2^{d+1}}\}$ and initialize $K$ Gaussians with mean $v_{(i)}$ and covariance $\mathbb{I}_{d+1}$ for $i \in \{1,\ldots,K\}$. Now, by means of the Gaussian Annulus Theorem (defined as in \cite[ Lemma $2.4$]{john2008}), the required sets $S_i$ are defined as annuli, i.e. $S_i=\{x\in\mathbb{R}^{d+1}\vert \sqrt{d}-\delta\leq \Vert x-v_{(i)} \Vert_2\leq \sqrt{d}+\delta \}$, for some $0<\delta< \sqrt{d}$,  where every $S_i$ contains the support of the $i$-th Gaussian  up to a fraction of $\frac{4}{\delta^2}e^{-\delta^2/4}$. Thus the initialized annuli form a union of disjoint compact $d+1$-connected manifolds.
	
	Lemma \ref{lem:mapsphere} conveys that for every given manifold $\mathcal{X}_i,$ there exist a smooth surjective function $f_i$, such that $f_i(S_i) = \mathcal{X}_i$. Thus, we obtain a collection of functions $\{f_i\}_{i=1}^K.$ Let $f:\mathbb{R}^{d+1}\rightarrow \mathbb{R}^{d}$ be of the form $f(x)=\sum_{i=1}^{K}f_i(x)\mathbbm{1}_{S_i}(x),$ where $\mathbbm{1}_{S_i}(x)$ is the indicator function defined on $S_i,$ i.e. $\mathbbm{1}_{S_i}(x)=1$ for $x\in S_i$ and 0 otherwise. Next, we  construct a function $f_{\rho}$ which is an approximation of $f$ via smoothing by convolution with a suitable function, e.g. the standard mollifier, defined by:
	
	\begin{align*}
	\eta_{\rho}&:\mathbb{R}^{d+1}\rightarrow \mathbb{R}, \\ \eta_{\rho}(x)&= \begin{dcases} 
	\frac{\alpha}{\rho^p}\exp\Bigg(-\frac{\rho^2}{\rho^2-\Vert x\Vert_2^2}\Bigg) & x\in \mathcal{B}^{\rho}_{d+1}(0)\\
	0 & \text{otherwise}
	\end{dcases},
	\end{align*}
	where $\rho>0$ and $\alpha$ is chosen, such that
	\begin{align*}
	\int_{\mathbb{R}^{d+1}} \eta_{\rho}(x)\text{d}x=1,
	\end{align*}
	and $\mathcal{B}^{\rho}_{d+1}(0)=\{x\in\mathbb{R}^{d+1}\vert \Vert x\Vert_2 <\rho\}$.
	
	Let $\Omega=\bigcup_i^K S_i\setminus\partial S_i$, where $\partial S_i=\{x\in\mathbb{R}^{d+1}\vert \Vert x-v_{(i)} \Vert_2 \in\{ \sqrt{d}-\delta, \sqrt{d}+\delta\} \}.$ Hence, $\Omega$ is a union of open bounded sets. Thus, the resulting convolved function $f_{\rho}$, defined by
	\begin{align*}
	f_{\rho}(x)&=f*\eta_{\rho}(x)=\int_{\Omega}\eta_{\rho}(y) f(x-y)\text{d}y\\&= \sum_{i=1}^{K}\int_{S_i\setminus\partial S_i}\eta_{\rho}(y) f_i(x-y)\mathbbm{1}_{S_i}(x-y)\text{d}y,
	\end{align*} 
	is continuous (e.g. \cite[Theorem $9.1.5$]{heil08real}) and it holds
	\begin{align*}
	\operatorname{supp}(f_{\rho})\subset \overline{\bigcup_{i\leq K} S_i+\mathcal{B}^{\rho}_{d+1}(0)}=\bigcup_{i\leq K} \overline{S_i+\mathcal{B}^{\rho}_{d+1}(0)}=:S,
	\end{align*}
	where  $S_i+\mathcal{B}^{\rho}_{d+1}(0)=\{x\in\mathbb{R}^{d+1}\vert \exists s\in S_i, \exists b\in \mathcal{B}^{\rho}_{d+1}(0),\,x=s+b\}$. We choose $\rho<\sqrt{d}-\delta$ such that the compact sets  $\overline{ S_i+\mathcal{B}^{\rho}_{d+1}(0)}$ become separable, in the sense that for all $i,j\in\{1,\ldots,K\}$, $\overline{ S_i+\mathcal{B}^{\rho}_{d+1}(0)} \cap \overline{ S_j+\mathcal{B}^{\rho}_{d+1}(0)}= \emptyset.$ For all $i\in\{1,\ldots,K\}$, $ s\in S_i$ and $b\in \mathcal{B}^{\rho}_{d+1}(0)$ it holds $\Vert s+b-v_{(i)}\Vert_2<2\sqrt{d}$.
 Recall that the means $v_{(i)},v_{(j)}$ for all $i,j\in\{1,\ldots,K\}$ are the vertices of $D$ defined above.  Moreover,
 \begin{align*}
 4\sqrt{d}&\leq \Vert v_{(i)} -v_{(j)}\Vert_2 \leq \Vert v_{(i)} -v_{(j)} - (s+b)+(s+b)\Vert_2\\
 &\leq \Vert v_{(i)} - (s+b)\Vert_2 +\Vert v_{(j)}-(s+b)\Vert_2\\&< 2\sqrt{d}+\Vert v_{(j)}-(s+b)\Vert_2.
 \end{align*}
 It follows that $\Vert v_{(j)}-(s+b)\Vert_2\geq 2\sqrt{d}>\sqrt{d}+\delta+\rho$ and,  therefore, $s+b \notin S_j+\mathcal{B}^{\rho}_{d+1}(0).$

	Now, define a cube $Z:=[-R,R]^{d+1}\subset \mathbb{R}^{d+1},$ where $R$ is chosen such that $S\subseteq Z.$ Thus, the function $f_{\rho}:Z\rightarrow\mathbb{R}^d$, $f_{\rho}(S_i)=\mathcal{X}_i$ for all $i\in\{1,\ldots,K\}$, fulfills the requirements of \cite[Theorem 5.1]{khrukov2019universality}. Therefore,  for all $\epsilon>0$ a neural network $G$ exists, such that for all $i \in\{1,\ldots,K\}$, $d_H(G(S_i),\mathcal{X}_i)<\epsilon$.
\end{proof}

Theorem \ref{theo:uni2}  gives theoretical guarantees for the existence of a generator $G$ and a disconnected latent space such that $G$ approximates the real data manifolds with small error. However, this holds only if the dimension of $\mathcal{X}$ is known, which is impossible to estimate in every real life application. Nevertheless this result and the discussion above provide a profound justification of the latent space choice. 

\subsection{Cluster Initialization}
To generate a collection of disjoint Gaussian clusters in a high-dimensional space for initializing the model, we propose the following heuristic. It is based on the idea of the annuli construction suggested in the proof of Theorem \ref{theo:uni2}. For this reason, consider the $d$-dimensional cube $[-1,1]^d,$ where the number of vertices equals $2^d.$  Let $V=\{v_1,\ldots,v_{M}\}$ with $M=2^d$ is the set of all vertices, i.e. for each $i\leq M,$ $v_i\in \{-1,1\}^d.$ We randomly sample a subset from $V$ of length $K$ and initialize the means of the Gaussian clusters. Here, we assume that the number of clusters $K$ does not exceed $2^d$, i.e. $K\leq 2^d.$ All the standard deviations $\sigma_k$ are initialized with values of at most 1.
Moreover, to avoid narrow Gaussian clusters with very small $\sigma_k$ for $k\in\{1,\ldots,K\}$, we set a lower bound of $0.1$ for the standard deviations.
 
\section{Experiments}
We fix the network structure, parameter initialization and use benchmark data to achieve a fair comparison between our new approach and compared existing models (GMVAE \cite{DilokthanakulMG16}, AAE \cite{makhzani2015aae}, ClusterGAN \cite{Mukherjee2018clustergan}). All models are trained with the Adam optimization method \cite{kingmaB14adam}, where $\beta_1=0.5$ and $\beta_2=0.99.$ The hyperparameter $\alpha$ is set to $1$ for all experiment. MMGAN and ClusterGAN generator and discriminator use the same architectures as AAE's decoder and encoder. Moreover, MMGAN/ClusterGAN discriminator and encoder have the same structure. For GMVAE we used an existing implementation\footnote{\url{https://github.com/psanch21/VAE-GMVAE}}. We evaluated our method  on a synthetic and three gray scale data sets (Moon \cite{scikit-learn}, MNIST \cite{lecun-mnisthandwrittendigit-2010}, Fashion MNIST \cite{fmnist} and Coil-20 \cite{Nene96columbiaobject}). For synthetic data, all model components have two dense layers with $128$ units per layer and ReLU activation function while for gray-scale images, MMGAN, ClusterGAN  and AAE nets have three CNN layers. The used activation function for the MMGAN/ClusterGAN discriminator and encoder is LeakyReLU.

 \begin{figure}[t]
	
	\begin{subfigure}[b]{0.31\linewidth}
		\centering
		\includegraphics[width=0.99\linewidth]{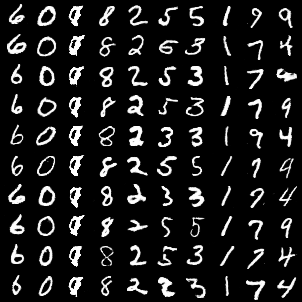}%
		\caption{SGAN}
		\label{fig:mnistsgan}
	\end{subfigure}
	\begin{subfigure}[b]{0.31\linewidth}
		\centering
		\includegraphics[width=0.99\linewidth]{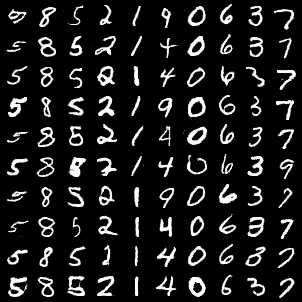}%
		\caption{RSGAN}
		\label{fig:mnistrsgan}
	\end{subfigure}
	\begin{subfigure}[b]{0.31\linewidth}
		\centering
		\includegraphics[width=0.99\linewidth]{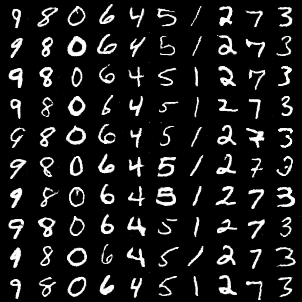}%
		\caption{cRaSGAN}
		\label{fig:mnistcrasgan}
	\end{subfigure}
	\caption[MNIST MMGAN Generated Output]{MNIST MMGAN fake images. Each column represents one cluster.}
	\label{fig:mnistmmgan}
\end{figure}
 

To study MMGAN functionalities, we conduct several experiments, where the model is trained on benchmarking datasets, and compared to other three competitive models. Table \ref{tab:moonres} provides an overview of the numerical results regarding the clustering performance on test data. We can conclude that MMGAN especially the cRaSGAN based one outperforms the other competitors in terms of all used evaluation measures: normalized mutual information (NMI), adjusted rank index (ARI), purity (ACC).
\begin{table}[t]
	\centering
	\resizebox{\columnwidth}{!}{%
		\begin{tabular}{|l|l|l|l|l|}
		\hline
		\textbf{Dataset}  & \textbf{Model}     & \textbf{NMI} & \textbf{ARI}& \textbf{ACC} \\ \hline
		\multirow{3}{*}{\begin{tabular}[c]{@{}l@{}}Moon\\ Data\end{tabular}} & 
		MMGAN (SGAN) & 0.64 & 0.74 & 0.93\\ \cline{2-5}  & 
		MMGAN (RSGAN) & \textbf{0.76} &\textbf{ 0.81} & \textbf{0.95}\\ \cline{2-5}  & 
		MMGAN (cRaSGAN)&  0.60& 0.70 & 0.92
		\\ \cline{2-5}  
		&GMVAE & 0.39 & 0.47 & 0.84\\  \cline{2-5} 
			&AAE & 0.48 & 0.57 & 0.86\\  \cline{2-5} 
		&ClusterGAN & 0.36 &0.45 & 0.83\\  \hline
		\multirow{3}{*}{MNIST}   
		& MMGAN (SGAN)  &       0.81       &       0.74 & 0.81       \\ \cline{2-5}                                                 
		& MMGAN (RSGAN)  &       0.86       &       0.80 & 0.88       \\ \cline{2-5} 
		& MMGAN (cRaSGAN) &       \textbf{0.92}     &      \textbf{0.94}  & \textbf{0.97} \\ \cline{2-5}    
		& GMVAE         &      0.68       &       0.57 & 0.70       \\ \cline{2-5} 
			&AAE & 0.73 & 0.64 & 0.76\\  \cline{2-5} 
		& ClusterGAN         &      0.72        &       0.64 & 0.74      \\ 
		\hline
		\multirow{3}{*}{\begin{tabular}[c]{@{}l@{}}Fashion\\ MNIST\end{tabular}} 
		
		& MMGAN (SGAN)  &    0.73   & 0.60 &0.74           \\ \cline{2-5} 
		& MMGAN (RSGAN) &      0.67       &      0.53 & 0.69 \\ \cline{2-5} 
		& MMGAN (cRaSGAN) &      \textbf{0.68}       &      \textbf{0.56} & \textbf{0.70} \\ \cline{2-5} 
		& GMVAE &      0.51      &      0.32 &0.48 \\ \cline{2-5} 
			&AAE & 0.52 & 0.31 & 0.50\\  \cline{2-5} 
		& ClusterGAN         &      0.30         & 0.18 & 0.42              \\    \hline
		\multirow{3}{*}{Coil-20} & 
		MMGAN (SGAN) & 0.71 & 0.55 & 0.67\\ \cline{2-5}  & 
		MMGAN (RSGAN) & 0.76 & 0.58 & 0.67\\ \cline{2-5}  & 
		MMGAN (cRaSGAN)&  \textbf{0.80}& \textbf{0.60} &\textbf{0.72}
		\\ \cline{2-5}  
		&GMVAE & 0.65 & 0.29 & 0.44\\  \cline{2-5} 
			&AAE & 0.75 & 0.56 & 0.61\\  \cline{2-5} 
		&ClusterGAN & 0.62 &0.38 & 0.45\\  \hline
		
	\end{tabular}}
	\caption[Benchmarking Results]{Comparison between MMGAN, GMVAE, AAE and ClusterGAN in terms of the clustering measures NMI, ARI and ACC.}
	\label{tab:moonres}
\end{table}

Figure \ref{fig:mnistmmgan} illustrates the generated output of SGAN, RSGAN, cRaSGAN - based MMGAN trained on the MNIST data. In this experiment, we can come to the  conclusion that the relativistic approach seems to be more stable than the SGAN one. For instance, it can be seen in Figure \ref{fig:mnistsgan} that the third cluster collapses and clusters $6$ and $7$ generate the same images. This is illustrated in the $6$th and $7$th column of Figure \ref{fig:mnistsgan}. 

In Figure \ref{fig:mnistmmganmean},  the heatmaps visualize cosine similarity between the cluster means, which is defined by $1-\cos(\mu_i,\mu_j)$ for two means $\mu_i$ and $\mu_j$, where $i,j\in\{1,\ldots,K\}$. The measure is bounded in $[-1,1]$. Two clusters are close to each other when the cosine similarity measure is around $1$. We also keep in mind that in high dimensional spaces two randomly sampled vectors are almost surely orthogonal.  
As we have pointed out earlier, clusters $6$ and $7$ in the SGAN based model generate the same mode. It can be expected that the cluster means form a small angle. However, Figure \ref{fig:mnistsganmean} does not support this hypothesis. The computed cosine measure is around $0$. The other two heatmaps (see Figures \ref{fig:mnistrsganmean} and \ref{fig:mnistcrasganmean}) also do not reveal any pattern between the similarity measurements and generated cluster output. For example, the digits $4,7$ and $9$  are often associated with one cluster (see column $6$ and $10$ in \ref{fig:mnistrsgan}). However, according to Figures \ref{fig:mnistrsganmean} and \ref{fig:mnistcrasganmean} the cluster means are not very similar. Herewith, we conclude that the obtained clusters do not allude to further structure in the latent space, i.e. when two latent clusters resemble in $\mathcal{Z},$ the corresponding generated output need not be similar in $\mathcal{X}.$
\begin{figure}[t]
	\begin{subfigure}[b]{0.31\linewidth}
		\centering
			\includegraphics[width=0.99\linewidth]{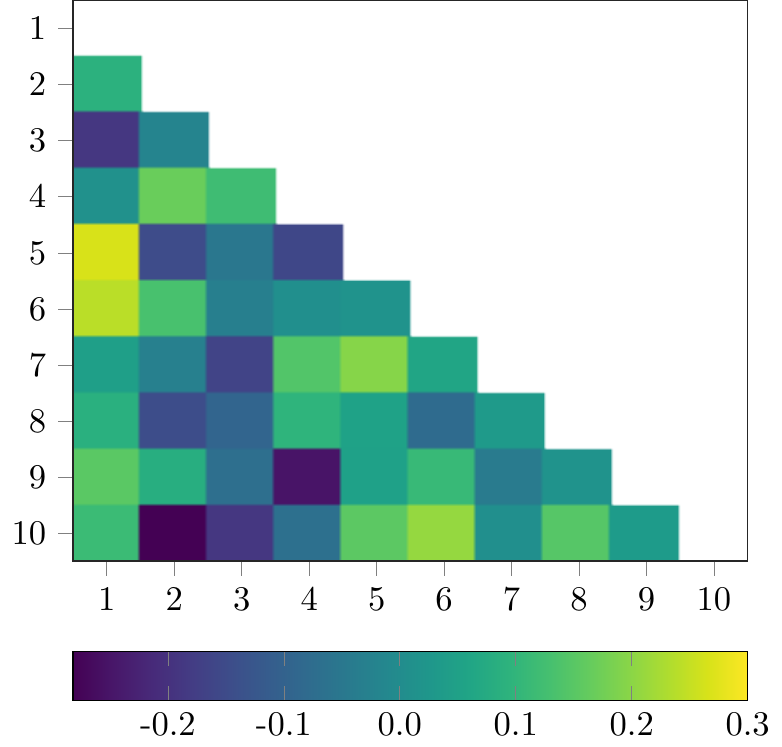}%
		\caption{SGAN}
		\label{fig:mnistsganmean}
	\end{subfigure}
	\begin{subfigure}[b]{0.31\linewidth}
		\centering
			\includegraphics[width=0.99\linewidth]{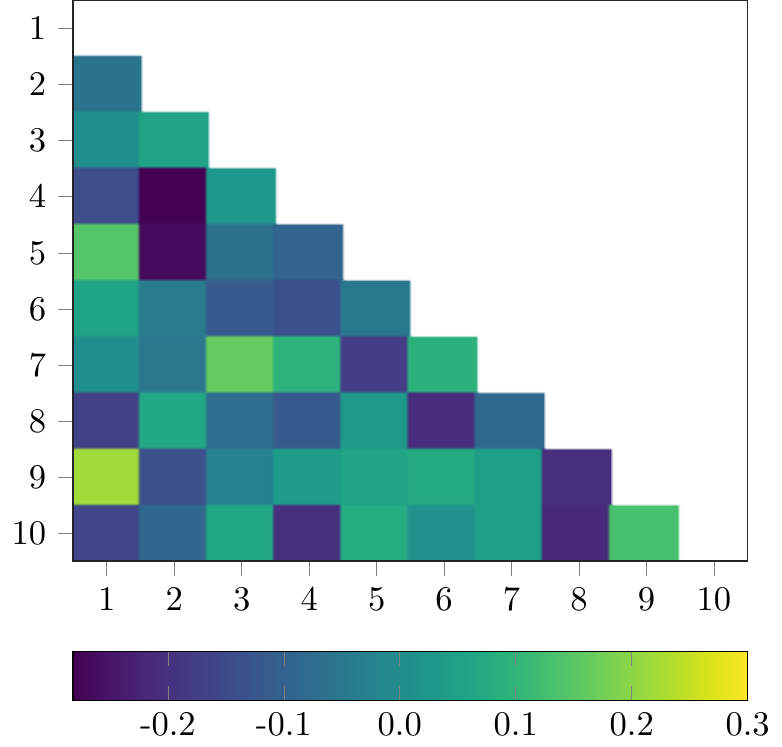}%
		\caption{RSGAN}
		\label{fig:mnistrsganmean}
	\end{subfigure}
	\begin{subfigure}[b]{0.31\linewidth}
		\centering
			\includegraphics[width=0.99\linewidth]{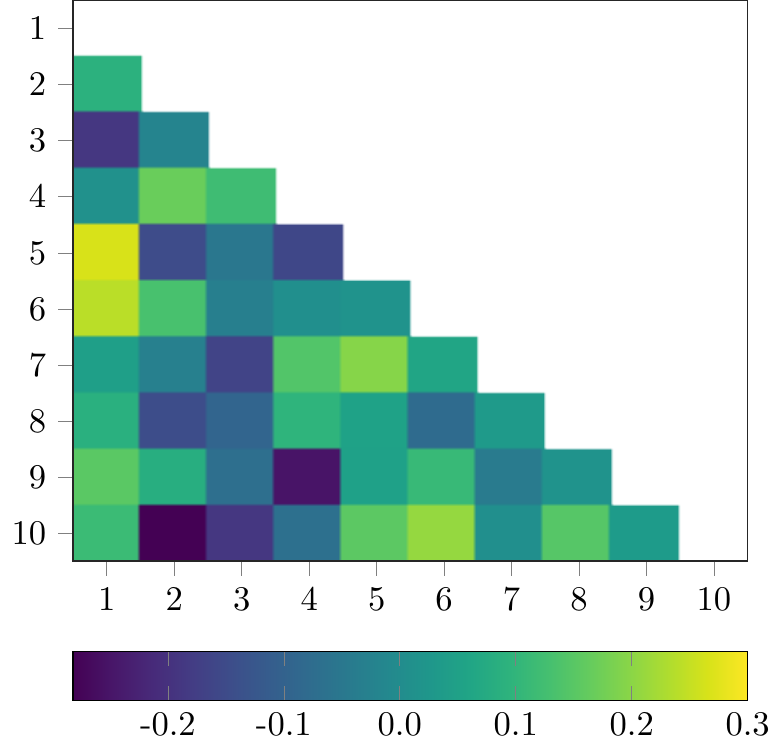}%
		\caption{cRaSGAN}
		\label{fig:mnistcrasganmean}
	\end{subfigure}
	\caption[MNIST MMGAN Mean Cosine Similarity]{MNIST MMGAN Mean Cosine Similarity. For every model a cosine similarity is computed between the cluster means.  The axes represent the clusters and the boxes refer to the pairwise distance between the means, such that high values are indicated by the yellow scale and low values by the blue one.}
	\label{fig:mnistmmganmean}
\end{figure}

In the next experiment,  we examine the effect of the pairings strategy on the MMGAN encoder performance. For each dataset (MNIST, Fashion MNIST, Coil-20) five MMGANs are trained using the cRaSGAN adversarial loss. For each dataset, the trained models are evaluated with respect to the clustering measures (NMI, ARI, ACC), which are summarized by their mean and standard deviations, shown in Table \ref{tab:pair}. 
\begin{table}[t]
	\centering
	\resizebox{\columnwidth}{!}{%
	\begin{tabular}{|l|c|l|l|l|}
		\hline
		\textbf{Dataset}  & \textbf{Model}     & \textbf{NMI} & \textbf{ARI}& \textbf{ACC} \\ \hline
		
		\multirow{2}{*}{MNIST}   
		
		& MMGAN  &        $\mathbf{0.91\pm 0.03}$     &      $\mathbf{0.90\pm 0.05}$  & $\mathbf{0.94\pm 0.04}$    \\ \cline{2-5} 
		& no pairings &       $0.92\pm 0.01$    &      $0.93 \pm 0.01$  & $0.96\pm 0.00$ \\   \hline

		\multirow{2}{*}{\begin{tabular}[c]{@{}l@{}}Fashion\\ MNIST\end{tabular}}

		& MMGAN &      $\mathbf{0.65\pm 0.01}$       &      $\mathbf{0.52\pm 0.06}$ & $\mathbf{0.66\pm 0.03}$ \\ \cline{2-5} 
		&no pairings &      $0.63\pm 0.01$      &      $0.49\pm 0.01$ &$0.63\pm 0.01$ \\  \hline
		
		\multirow{2}{*}{Coil-20} & 
		
		MMGAN&  $\mathbf{0.80\pm 0.01}$& $\mathbf{0.62\pm 0.02}$ & $\mathbf{0.73 \pm 0.03}$
		\\ \cline{2-5}  
		&no pairings& $0.70\pm 0.02$ & $0.48\pm 0.04$ & $0.64\pm 0.02$\\   \hline

	\end{tabular}}
	\caption[Pairing Strategy Verification for the Clustering Performance]{Comparison between the MMGAN framework and the one without utilizing the pairing strategy in terms of encoder clustering performance.  }
	\label{tab:pair}
\end{table}
Analogously to the setting above, we train MMGANs without using the pairing strategy, i.e. the  input noise is sampled randomly from the Gaussian mixture model. Thus, the formed pairings $(x_r,x_f)$ do not necessarily refer to the same cluster.  Table \ref{tab:pair} provides the clustering summary statistics for this type of model, as well.

\begin{figure}[t]
	\centering
	\includegraphics[scale=0.8]{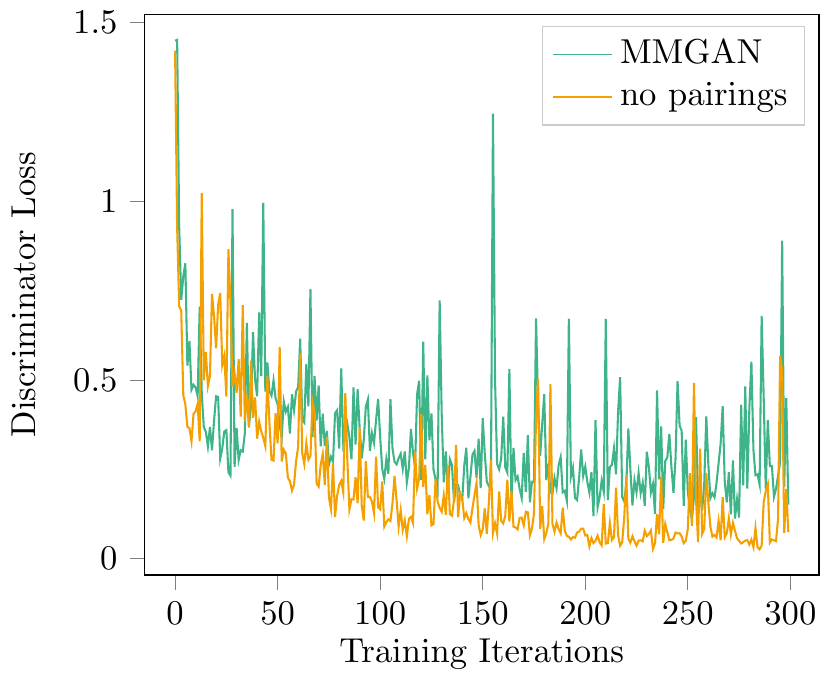}%
	\caption[Coil-20 MMGAN Discriminator Loss]{Here, we compare the discriminator loss of the MMGAN model and the one trained without the pairing strategy. Both models are fitted on the Coil-20 dataset.}
	\label{fig:pairs}
\end{figure}
By considering Table \ref{tab:pair}, we can conclude that for both the MNIST and Fashion MNIST data the two type of models show similar results in terms of clustering performance. Moreover, in the Coil-20 dataset case, our proposed MMGAN framework outperforms the other one. This observation is also supported by Figure \ref{fig:pairs}, which illustrates the discriminator loss over the training iterations for both type of models. It can be concluded that the random strategy pushes the discriminator loss faster to $0$ than the pairings one.

 \begin{figure}[t]
 	\begin{subfigure}[b]{0.49\linewidth}
 		\centering
 		\includegraphics[width=0.9\linewidth]{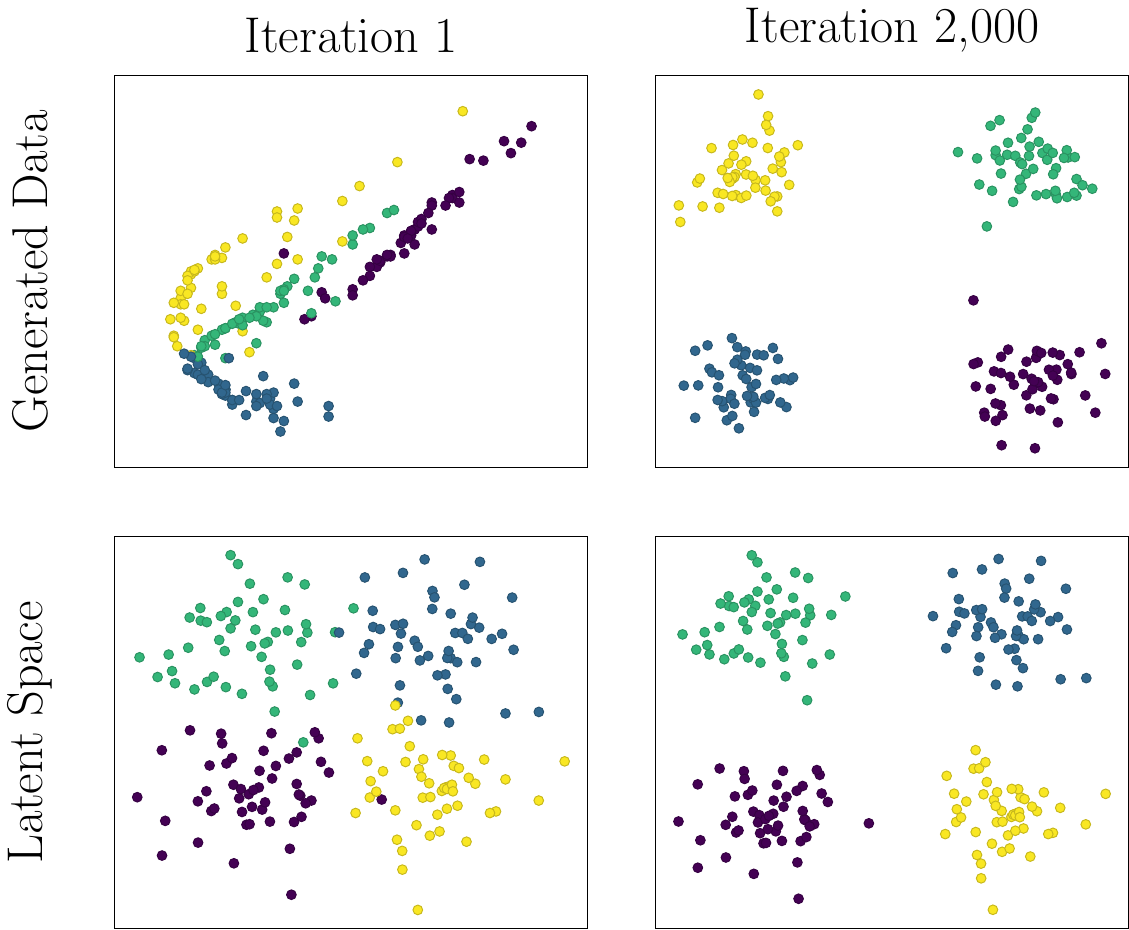}
 		\caption{Disjoint Initialization}
 		\label{fig:syntheu}
 	\end{subfigure}
 	\begin{subfigure}[b]{0.49\linewidth}
 		\centering
 		\includegraphics[width=0.9\linewidth]{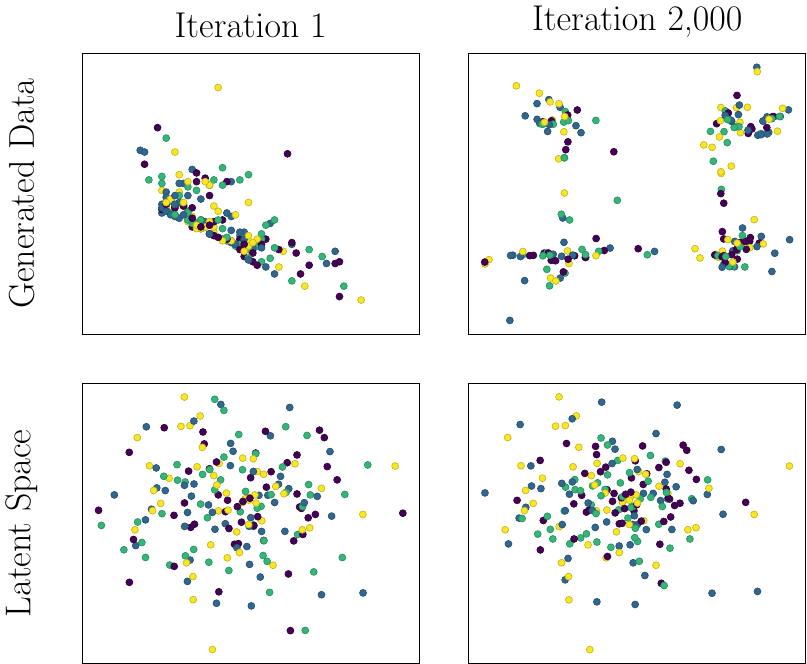}
 		\caption{Overlapping Initialization}
 		\label{fig:syntrand}
 	\end{subfigure}
 	\caption[Parameter Initialization]{The first row shows a sample of the generated data points over the training iterations, i.e. after the first and $2,000$. The second row represents the learned latent space distribution w.r.t the initialization. The coloring refers to the given prior clusters. The model with fixed initialization cannot reconstruct the predefined clusters.}
 	\label{fig:synexp1}
 \end{figure}
In our experiment, depicted in Figure \ref{fig:synexp1}, we trained two MMGANs with different initialization of the cluster means. The first one (see Figure \ref{fig:syntheu}) employs the heuristic explained above, while the second one uses the same starting value for each cluster mean. For both experiments the features standard deviations have starting values $0.5.$ 
The latent codes used for acquiring the data points are fixed over all iterations.
Interestingly, only after the first iteration  a clustered structure can be recognized in the generated data. It can be also seen that the generator manages to reconstruct the initial data distribution and the encoder successfully performs the unsupervised clustering task by achieving maximal evaluation scores. The second row of  Figure \ref{fig:syntheu} shows the latent space parameter learning over time. It can be observed that the cluster specific standard deviations decrease.  Figure \ref{fig:syntrand} similarly to Figure \ref{fig:syntheu} shows the generator behavior over time, yet for overlapping clustering initialization. It indicates that the resulting encoder does not match the prior labeling. Moreover, the sample quality impairs compared to the first MMGAN.
 
\section{Conclusions}
This paper introduced a new model from the VAE-GANs hybrid family, which is adapted for both inference learning and approximating real data distributions with disconnected support. Throughout presenting theoretical and empirical results, we have justified the specific structure of our model. We observe that MMGAN excels in the generative modeling task and successfully clusters the real data in the latent space, regarding the labels in the used datasets. In the conducted experiments, we observed an outstanding performance compared to  other two state-of-the-art models which are related to this field.
\bibliographystyle{aaai} \bibliography{lib}

\begin{thebibliography}{}

\bibitem[\protect\citeauthoryear{Arjovsky and
  Bottou}{2017}]{arjovsky2017manifolds}
Arjovsky, M., and Bottou, L.
\newblock 2017.
\newblock {Towards Principled Methods for Training Generative Adversarial
  Networks}.
\newblock In {\em 5th International Conference on Learning Representations,
  {ICLR} 2017, Toulon, France, April 24-26, 2017, Conference Track
  Proceedings}.

\bibitem[\protect\citeauthoryear{Arjovsky, Chintala, and
  Bottou}{2017}]{arjovsky2017wasserstein}
Arjovsky, M.; Chintala, S.; and Bottou, L.
\newblock 2017.
\newblock {Wasserstein Generative Adversarial Networks}.
\newblock In {\em International Conference on Machine Learning},  214--223.

\bibitem[\protect\citeauthoryear{Ben{-}Yosef and
  Weinshall}{2018}]{yosef2018gmgan}
Ben{-}Yosef, M., and Weinshall, D.
\newblock 2018.
\newblock {Gaussian Mixture Generative Adversarial Networks for Diverse
  Datasets, and the Unsupervised Clustering of Images}.
\newblock {\em arXiv}.
\newblock \url{http://arxiv.org/abs/1808.10356}.

\bibitem[\protect\citeauthoryear{Blum, Hopcroft, and
  Kannan}{2015}]{blum2015foundations}
Blum, A.; Hopcroft, J.; and Kannan, R.
\newblock 2015.
\newblock {\em {Foundations of Data Science}}.
\newblock {Cambridge University Press}.

\bibitem[\protect\citeauthoryear{Chen \bgroup et al\mbox.\egroup
  }{2016}]{chen2016infogan}
Chen, X.; Duan, Y.; Houthooft, R.; Schulman, J.; Sutskever, I.; and Abbeel, P.
\newblock 2016.
\newblock {InfoGAN: Interpretable Representation Learning by Information
  Maximizing Generative Adversarial Nets}.
\newblock {\em arXiv}.
\newblock \url{http://arxiv.org/abs/1606.03657}.

\bibitem[\protect\citeauthoryear{Cybenko}{1989}]{Cybenko89}
Cybenko, G.
\newblock 1989.
\newblock {Approximation by Superpositions of a Sigmoidal Function}.
\newblock {\em {Mathematics of control, signals and systems }} 2(4):303--314.

\bibitem[\protect\citeauthoryear{Dilokthanakul \bgroup et al\mbox.\egroup
  }{2016}]{DilokthanakulMG16}
Dilokthanakul, N.; Mediano, P. A.~M.; Garnelo, M.; Lee, M. C.~H.; Salimbeni,
  H.; Arulkumaran, K.; and Shanahan, M.
\newblock 2016.
\newblock {Deep Unsupervised Clustering with Gaussian Mixture Variational
  Autoencoders}.
\newblock {\em arXiv}.
\newblock \url{http://arxiv.org/abs/1611.02648}.

\bibitem[\protect\citeauthoryear{Goodfellow \bgroup et al\mbox.\egroup
  }{2014}]{goodfellow2014generative}
Goodfellow, I.; Pouget-Abadie, J.; Mirza, M.; Xu, B.; Warde-Farley, D.; Ozair,
  S.; Courville, A.; and Bengio, Y.
\newblock 2014.
\newblock {Generative Adversarial Nets}.
\newblock In {\em {Advances in Neural Information Processing Systems}},
  2672--2680.

\bibitem[\protect\citeauthoryear{Gurumurthy, Sarvadevabhatla, and
  Babu}{2017}]{gurumurthy2017deligan}
Gurumurthy, S.; Sarvadevabhatla, R.~K.; and Babu, R.~V.
\newblock 2017.
\newblock {DeLiGAN: Generative Adversarial Networks for Diverse and Limited
  Data}.
\newblock In {\em 2017 {IEEE} Conference on Computer Vision and Pattern
  Recognition, {CVPR} 2017, Honolulu, HI, USA, July 21-26, 2017},  4941--4949.

\bibitem[\protect\citeauthoryear{Heil}{2019}]{heil08real}
Heil, C.
\newblock 2019.
\newblock {\em Introduction to Real Analysis}.
\newblock Graduate texts in mathematics. Springer International Publishing.

\bibitem[\protect\citeauthoryear{Hopcroft and Kannan}{2013}]{john2008}
Hopcroft, J., and Kannan, R.
\newblock 2013.
\newblock {\em {Foundations of Data Science}}.
\newblock New York: Cornell University.

\bibitem[\protect\citeauthoryear{Hornik}{1991}]{Hornik91}
Hornik, K.
\newblock 1991.
\newblock {Approximation Capabilities of Multilayer Feedforward Networks}.
\newblock {\em Neural Networks} 4(2):251--257.

\bibitem[\protect\citeauthoryear{Jolicoeur{-}Martineau}{2018}]{jolicoeur2018relativistic}
Jolicoeur{-}Martineau, A.
\newblock 2018.
\newblock {The relativistic discriminator: a key element missing from standard
  GAN}.
\newblock {\em arXiv}.
\newblock \url{http://arxiv.org/abs/1807.00734}.

\bibitem[\protect\citeauthoryear{{Jost, J{\"u}rgen}}{2008}]{jost2008riemannian}
{Jost, J{\"u}rgen}.
\newblock 2008.
\newblock {\em {Riemannian Geometry and Geometric Analysis}}, volume 42005.
\newblock Springer-Verlag, Berlin Heidelberg.

\bibitem[\protect\citeauthoryear{Khayatkhoei, Singh, and
  Elgammal}{2018}]{Khayatkhoei18Disconnected}
Khayatkhoei, M.; Singh, M.~K.; and Elgammal, A.
\newblock 2018.
\newblock {Disconnected Manifold Learning for Generative Adversarial Networks}.
\newblock In {\em Advances in Neural Information Processing Systems 31: Annual
  Conference on Neural Information Processing Systems 2018, NeurIPS 2018, 3-8
  December 2018, Montr{\'{e}}al, Canada.},  7354--7364.

\bibitem[\protect\citeauthoryear{Khrulkov and
  Oseledets}{2019}]{khrukov2019universality}
Khrulkov, V., and Oseledets, I.
\newblock 2019.
\newblock {Universality Theorems for Generative Models}.
\newblock {\em arXiv}.
\newblock \url{http://arxiv.org/abs/1905.11520}.

\bibitem[\protect\citeauthoryear{Kingma and Ba}{2015}]{kingmaB14adam}
Kingma, D.~P., and Ba, J.
\newblock 2015.
\newblock {Adam: {A} Method for Stochastic Optimization}.
\newblock In {\em 3rd International Conference on Learning Representations,
  {ICLR} 2015, San Diego, CA, USA, May 7-9, 2015, Conference Track
  Proceedings}.

\bibitem[\protect\citeauthoryear{Kingma and Welling}{2013}]{kingma2013auto}
Kingma, D.~P., and Welling, M.
\newblock 2013.
\newblock {Auto-encoding Variational Bayes}.
\newblock {\em arXiv}.
\newblock \url{http://arxiv.org/abs/1312.6114}.

\bibitem[\protect\citeauthoryear{LeCun and
  Cortes}{2010}]{lecun-mnisthandwrittendigit-2010}
LeCun, Y., and Cortes, C.
\newblock 2010.
\newblock {MNIST} handwritten digit database.
\newblock \url{http://yann.lecun.com/exdb/mnist/}.

\bibitem[\protect\citeauthoryear{Makhzani \bgroup et al\mbox.\egroup
  }{2015}]{makhzani2015aae}
Makhzani, A.; Shlens, J.; Jaitly, N.; and Goodfellow, I.~J.
\newblock 2015.
\newblock {Adversarial Autoencoders}.
\newblock {\em arXiv}.
\newblock \url{http://arxiv.org/abs/1511.05644}.

\bibitem[\protect\citeauthoryear{Mukherjee \bgroup et al\mbox.\egroup
  }{2018}]{Mukherjee2018clustergan}
Mukherjee, S.; Asnani, H.; Lin, E.; and Kannan, S.
\newblock 2018.
\newblock {ClusterGAN : Latent Space Clustering in Generative Adversarial
  Networks}.
\newblock {\em arXiv}.
\newblock \url{http://arxiv.org/abs/1809.03627}.

\bibitem[\protect\citeauthoryear{Munkres}{2017}]{munkres2017topology}
Munkres, J.
\newblock 2017.
\newblock {\em Topology}.
\newblock Math Classics. Pearson.

\bibitem[\protect\citeauthoryear{Nagarajan and
  Kolter}{2017}]{nagarajan2017gradient}
Nagarajan, V., and Kolter, J.~Z.
\newblock 2017.
\newblock {Gradient descent GAN optimization is locally stable}.
\newblock In {\em Advances in Neural Information Processing Systems},
  5585--5595.

\bibitem[\protect\citeauthoryear{Nene, Nayar, and
  Murase}{1996}]{Nene96columbiaobject}
Nene, S.~A.; Nayar, S.~K.; and Murase, H.
\newblock 1996.
\newblock {Columbia Object Image Library (COIL-20)}.
\newblock Technical report, Columbia University.

\bibitem[\protect\citeauthoryear{Pedregosa \bgroup et al\mbox.\egroup
  }{2011}]{scikit-learn}
Pedregosa, F.; Varoquaux, G.; Gramfort, A.; Michel, V.; Thirion, B.; Grisel,
  O.; Blondel, M.; Prettenhofer, P.; Weiss, R.; Dubourg, V.; Vanderplas, J.;
  Passos, A.; Cournapeau, D.; Brucher, M.; Perrot, M.; and Duchesnay, E.
\newblock 2011.
\newblock {Scikit-learn: Machine Learning in {P}ython}.
\newblock {\em Journal of Machine Learning Research} 12:2825--2830.

\bibitem[\protect\citeauthoryear{Xiao, Rasul, and Vollgraf}{2017}]{fmnist}
Xiao, H.; Rasul, K.; and Vollgraf, R.
\newblock 2017.
\newblock {Fashion-MNIST: a Novel Image Dataset for Benchmarking Machine
  Learning Algorithms}.
\newblock {\em arXiv}.
\newblock \url{http://arxiv.org/abs/1708.07747}.

\end{thebibliography}
\end{document}